\documentclass[final, 1p, authoryear, 12pt]{elsarticle}
\usepackage{hyperref}
\hypersetup{colorlinks=true,citecolor=blue}
\RequirePackage{amsthm,multirow,fancybox,amssymb}
\usepackage[fleqn]{amsmath}
\RequirePackage{lastpage}
\usepackage{lineno}
\usepackage{geometry}
\usepackage[usenames]{color}
\usepackage{float}
\floatstyle{ruled}

\newfloat{algorithm}{tbp}{loa}
\providecommand{\algorithmname}{Algorithm}
\floatname{algorithm}{\protect\algorithmname}

\theoremstyle{plain}
\newtheorem{thm}{\protect\theoremname}
  \theoremstyle{plain}
  \newtheorem{lem}[thm]{\protect\lemmaname}

\usepackage[english]{babel}
  \providecommand{\lemmaname}{Lemma}
\providecommand{\theoremname}{Theorem}

\journal{Computational Statistics and Data Aanlysis}

\begin{document}

\begin{frontmatter}

\title{A note on the lack of symmetry in the graphical lasso}

\author[icme]{Benjamin T. Rolfs\corref{cor1}}
\ead{benrolfs@stanford.edu}
\author[stats]{Bala Rajaratnam}
\ead{brajarat@stanford.edu}
\cortext[cor1]{Corresponding author}
\address[icme]{Inst. for Computational and Applied Mathematics, Stanford University, Stanford, CA 94395}
\address[stats]{Department of Statistics, Stanford University, Stanford, CA 94305}

\begin{abstract}
{The graphical lasso (glasso) is a widely-used
fast algorithm for estimating sparse inverse covariance matrices.
The glasso solves an $\ell_{1}$ penalized maximum likelihood problem
and is available as an \texttt{R} library on CRAN. The output from the glasso, a regularized covariance matrix estimate 
$\hat{\Sigma}_{glasso}$ and a
sparse inverse covariance matrix estimate $\hat{\Omega}_{glasso}$,
not only identify a graphical model but can also serve as intermediate
inputs into multivariate procedures such as PCA, LDA, MANOVA, and others.
The glasso indeed produces a covariance matrix estimate $\hat{\Sigma}_{glasso}$ which solves the $\ell_1$ penalized optimization problem in a dual sense; however, the method for producing $\hat{\Omega}_{glasso}$ after this optimization is inexact and may produce asymmetric estimates. This problem is exacerbated when the amount of
$\ell_{1}$ regularization that is applied is small, which in turn is more likely to occur if the true 
underlying inverse covariance matrix is not sparse. The lack of symmetry
can potentially have consequences. First, it implies that $\hat{\Sigma}_{glasso}^{-1}\neq\hat{\Omega}_{glasso}$
and second, asymmetry can possibly lead to negative or complex eigenvalues,
rendering many multivariate procedures which may depend on $\hat{\Omega}_{glasso}$
unusable. We demonstrate this problem, explain its causes, and propose
possible remedies.}
\end{abstract}


\end{frontmatter}

%\linenumbers

\section{Introduction}
\label{Sec:intro}

In modern applications, many data sets are simultaneously high-dimensional
and low in sample size. Classic examples include microarray gene expression
and SNP data. Dealing with such datasets has become an area of great
interest in many fields such as biostatistics. Algorithms such as the graphical lasso
\citep{Friedman08, Hastie09} have been proposed to obtain regularized covariance
estimators in the $n\ll p$ setting (where $n$ is the sample size
and $p$ is the problem dimension) as well as perform graphical model
selection. 

In the case of the graphical lasso, graphical model selection involves
inferring a concentration graph (or equivalently, a Markov model).
A concentration graph encodes zeros in the inverse covariance (concentration)
matrix, i.e., 
$i\not\not\sim j$ for $i,j\in\left\{ 1,\dots,p\right\} $
in the graph implies that the partial correlation 
$\rho\left(X_{i},X_{j}|X_{k\notin\left\{ i,j\right\} }\right)=0$.
Along with inferring such a graph, the glasso provides $p\times p$
dimensional matrix estimators for both the covariance and concentration
matrices, denoted 
$\hat{\Sigma}_{\lambda}$ and $\hat{\Omega}_{\lambda}$
respectively, for a given penalty parameter $\lambda>0$. In particular,
$\hat{\Omega}_{\lambda}$ is the solution to the convex maximization problem 
\begin{alignat}{1}
	\hat{\Omega}_{\lambda}=\hat{\Sigma}_{\lambda}^{-1} & =\arg\min_{X\succ0}\ 			\left[\log\det\left(X\right)-\mbox{tr}\left(SX\right)-\lambda\left\Vert X\right\Vert 	_{1}\right]\label{eq:glassoPrimal}
\end{alignat}
where $S$ is the sample covariance matrix, 
$X=\left\{ x_{ij}\right\} _{i,j=1}^{p}$
is positive definite and 
$\left\Vert X\right\Vert _{1}=\sum_{i,j}\left|x_{ij}\right|$.
The non-zero elements of 
$\hat{\Omega}_{\lambda}$ 
correspond to edges in the estimated concentration graph. 

In some applications, graphical model selection is the primary goal,
where in other situations the estimators $\hat{\Sigma}_{\lambda}$
and $\hat{\Omega}_{\lambda}$ are used as inputs into other multivariate
algorithms where a regularized covariance estimator is required. Typical
examples include LDA, PCA, and MANOVA. Hence, it is often necessary that not only 
$\hat{\Sigma}_{\lambda}^{T}=\hat{\Sigma}_{\lambda}\succ0$,
but also that 
$\hat{\Omega}_{\lambda}^{T}=\hat{\Omega}_{\lambda}$,
$\hat{\Omega}_{\lambda}\succ0$, 
and 
$\hat{\Omega}_{\lambda}^{-1}=\hat{\Sigma}_{\lambda}$.
We find that the output of the graphical lasso does not meet these
conditions in certain situations, explain why, and discuss how to
solve this problem. Such situations arise primarily when $S$ is rank-deficient
and $\lambda$ is small. A low level of regularization is required when the true underlying concentration matrix is not sparse. It should however be noted that the glasso algorithm
does indeed solve the dual problem corresponding to \eqref{eq:glassoPrimal},
so the above assertions should be interpreted in context.

\section{Motivating Examples}
\label{Sec:examples}

We now present two motivating examples, one in a classical setting
and another in a high-dimensional setting, to illustrate the problem.

\subsection{Example 1: Low dimensional, large sample size inverse covariance
estimation}

Consider $n=500$ \textit{i.i.d.} samples drawn from a $p=5$
dimensional multivariate Gaussian distribution with mean $\mu=0$
and concentration matrix:
\begin{alignat*}{1}
	\Omega & =
	\left[\begin{array}{ccccc}
	\ 2.425 & \ 0.069 & -0.885 & 0 & 0\\
	\ 0.069 & \ 2.944 & -0.129 & 0.988 & 0\\
	-0.885 & -0.129 & \ 2.696 & 0.035 & -0.974\\
	0 & \ 0.988 & \ 0.035 & 1.724 & \ 0.851\\
	0 & 0 & -0.974 & 0.851 & \ 1.000
	\end{array}\right]
\end{alignat*}

The glasso algorithm was applied to this data set. A regularization
parameter of $\lambda=0.0033$, which is close to the cross-validated
estimate, was chosen to demonstrate the problem. The glasso estimators
for $\Omega$ and $\Sigma=\Omega^{-1}$ for a given $\lambda$ are denoted 
$\hat{\Omega}_{\lambda}$ and $\hat{\Sigma}_{\lambda}$. 

For reasons which are clarified in Section \ref{Sec:causes}, the glasso produces
estimators which are neither symmetric nor true inverses of one another,
i.e., 
$\hat{\Omega}_{\lambda}^{T}\neq\hat{\Omega}_{\lambda}$ and
$\hat{\Sigma}_{\lambda}^{-1}\neq\hat{\Omega}_{\lambda}$. 
To quantify
the lack of symmetry, consider the matrix of relative errors between
the elements of $\hat{\Omega}_{\lambda}$ and $\hat{\Omega}_{\lambda}^{T}$,
as defined by 
$Err_{ij}=100\left|\frac{\hat{\Omega}_{\lambda}\left(i,j\right)-\hat{\Omega}_{\lambda}^{T}\left(i,j\right)}{\hat{\Omega}_{\lambda}\left(i,j\right)}\right|\%$.

For the numerical example above, 
\begin{alignat*}{2}
	Err & =\left[\begin{array}{ccccc}
	0 & 1.94 & 0.05 & 0 & 0.25\\
	1.98 & 0 & 2.84 & 0.04 & \infty\\
	0.05 & 2.77 & 0 & 0.88 & 0.04\\
	0 & 0.04 & 0.89 & 0 & 0.01\\
	0.25 & 100.00 & 0.04 & 0.01 & 0
	\end{array}\right]
\end{alignat*}
with the convention that if $\hat{\Omega}_{\lambda}\left(i,j\right)=0=\hat{\Omega}_{\lambda}\left(j,i\right)$
then $Err_{ij}=0$. Note that the entries $Err_{5,2}=100\%$ and $Err_{2,5}=\infty$
occur because $\hat{\Omega}_{\lambda}\left(5,2\right)\neq0$ while
$\hat{\Omega}_{\lambda}\left(2,5\right)=0$. 

Although the relative errors are small, i.e., on the order of $2\%$,
there is a clear lack of symmetry in $\hat{\Omega}_{\lambda}$ and
moreover the sparsity patterns in the upper and lower parts of $\hat{\Omega}_{\lambda}$
are different, and thus yield two different graphical models. In particular,
$\hat{\Omega}_{\lambda}\left(5,2\right)\neq0$ which indicates an
edge between variables 2 and 5, while $\hat{\Omega}_{\lambda}\left(2,5\right)=0$
indicates the absence of such. Furthermore, in high-dimensional examples,
a graph is often calculated automatically when $\left(\hat{\Omega}_{\lambda}\right)_{ij}>\epsilon$
for some small $\epsilon$. In such cases, a lack of symmetry may
result, yielding two separate graphs.

\subsection{Example 2: High dimensional, low sample size autoregressive model}

The lack of symmetry in $\hat{\Omega}_{\lambda}$, and the resulting
difference in the concentration graphs corresponding to the upper
and lower parts of $\hat{\Omega}_{\lambda}$, often becomes more pronounced
as the dimension $p$ grows. 

We now consider a high dimensional example with $n=250$ \textit{i.i.d.}
samples drawn from a Gaussian $\mbox{AR}(1)$ model such that $X_{t+1}=\phi X_{t}+\epsilon_{t}$
for $t=2,...,p$ and $X_{1}=\epsilon_{1}$. Here, $p=500$, $\phi=0.75$,
and $\epsilon_{t}\overset{i.i.d}{\sim}\mathcal{N}\left(0,1\right)$,
$t=1,\dots,p$. The concentration matrix $\Omega$ is tridiagonal,
with the diagonal entries equal to 1 and the off-diagonal entries
equal to $-0.75$. 

Given a glasso estimator $\hat{\Omega}_{\lambda}$, let $E_{1}$ and
$E_{2}$ denote the edge sets corresponding to the upper and lower
halves of $\hat{\Omega}_{\lambda}$, respectively. Then the symmetric
difference $\left|E_{1}\Delta E_{2}\right|$ is the number of edges
which are present in the concentration graph encoded by one half of $\hat{\Omega}_{\lambda}$
but not in the graph encoded by the other half. 

The glasso algorithm was applied to samples from the above model with
the regularization parameter $\lambda$ taking values between $0.001$
and $0.03$ in increments of $0.001$. To put these values in perspective,
note that when $\lambda=0.03$, $102,278$ out of $124,750$ $\left(82\%\right)$
of the estimated off-diagonal entries were $0$. The number of edge
differences $\left|E_{1}\Delta E_{2}\right|$ corresponding to $\hat{\Omega}_{\lambda}$
as $\lambda$ varies between $0.001$ and $0.03$ is shown in Figure \ref{fig:glassoErrors_fig1}. 

\begin{figure}[h]
	\centering
	\caption{$\left|E_{1}\Delta E_{2}\right|$ vs. $\lambda$ for an $AR(1)$ model
			 with $\phi=0.75$, $p=500$, and $n=250$. The red dashed line is
		     at $\left|E_{1}\Delta E_{2}\right|=20$.\label{fig:glassoErrors_fig1} }
	\includegraphics{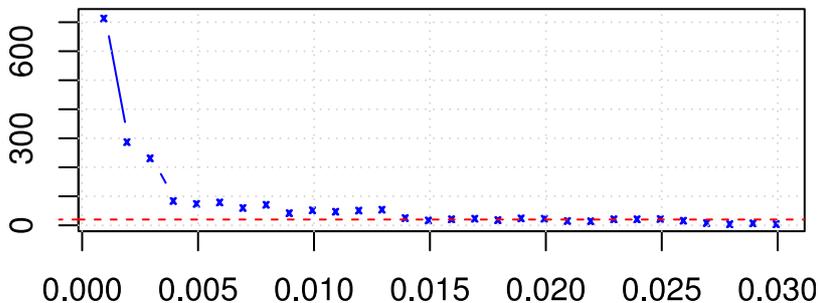}
\end{figure}

Note that at small values of $\lambda$, the difference in the graphs
corresponding to the upper and lower parts of $\hat{\Omega}_{\lambda}$
as denoted by $\left|E_{1}\Delta E_{2}\right|$ can be substantial.
Hence, the lack of symmetry in $\hat{\Omega}_{\lambda}$ can result
in two completely different graphical models. Moreover, although$\left|E_{1}\Delta E_{2}\right|$
decreases as $\lambda$ increases it nevertheless remains nonzero
as regularization increases.

\subsection{Consequences of asymmetry in the glasso concentration matrix estimator }

Users of the glasso may find the lack of symmetry a problem for a
number of reasons:
\begin{enumerate}
\item $\hat{\Omega}_{\lambda}$ is not a mathematically valid estimator
	for $\hat{\Omega}$, since $\hat{\Omega}_{\lambda}^{T}\neq\hat{\Omega}_{\lambda}$
	and $\hat{\Omega}_{\lambda}\neq\hat{\Sigma}_{\lambda}^{-1}$. 
\item There is no guarantee that $\hat{\Omega}_{\lambda}$ has real positive
	eigenvalues. If it has negative or complex eigenvalues, many multivariate
	procedures such as LDA and PCA may not be well-defined.
\item There may be differences between the edge sets of the concentration
	graphs corresponding to the respective upper and lower halves of $\hat{\Omega}_{\lambda}$. 
\end{enumerate}
We examine the causes of the lack of symmetry in Section \ref{Sec:causes} and suggest
possible remedies in Section \ref{Sec:solutions}. 

\section{Cause of Asymmetry in the Glasso Concentration Matrix Estimator}
\label{Sec:causes}

The glasso algorithm, taken directly from \cite{Hastie09}, is shown
in Algorithm \ref{Alg:glasso_alg1}. For further details concerning the glasso and its
convergence, see \cite{Friedman08} and \cite{Hastie09}. In Algorithm \ref{Alg:glasso_alg1}, $S$ is the
sample covariance matrix, $\lambda$ is the glasso penalty parameter,
and $W$ is a matrix on which the glasso iterates. In Step 2 of Algorithm
\ref{Alg:glasso_alg1}, $W_{11}$ refers to the submatrix of $W$ without its $j^{th}$
row and column, and $s_{12}$ is the $j^{th}$ column of the sample
covariance matrix without the diagonal element $s_{jj}$. In Step
3 of Algorithm \ref{Alg:glasso_alg1}, $\hat{\theta}_{12}$ for a given $j$ is the
$j^{th}$ column of the matrix $\Theta$ without $\Theta_{jj}$. Upon
termination of the algorithm, the current iterate $W$ is set to $\hat{\Sigma}_{\lambda}$
and $\Theta$ is set to $\hat{\Omega}_{\lambda}$, and referred to
as the glasso estimators. 

\begin{algorithm}[h]
	\centering
	\includegraphics[width=4.5in]{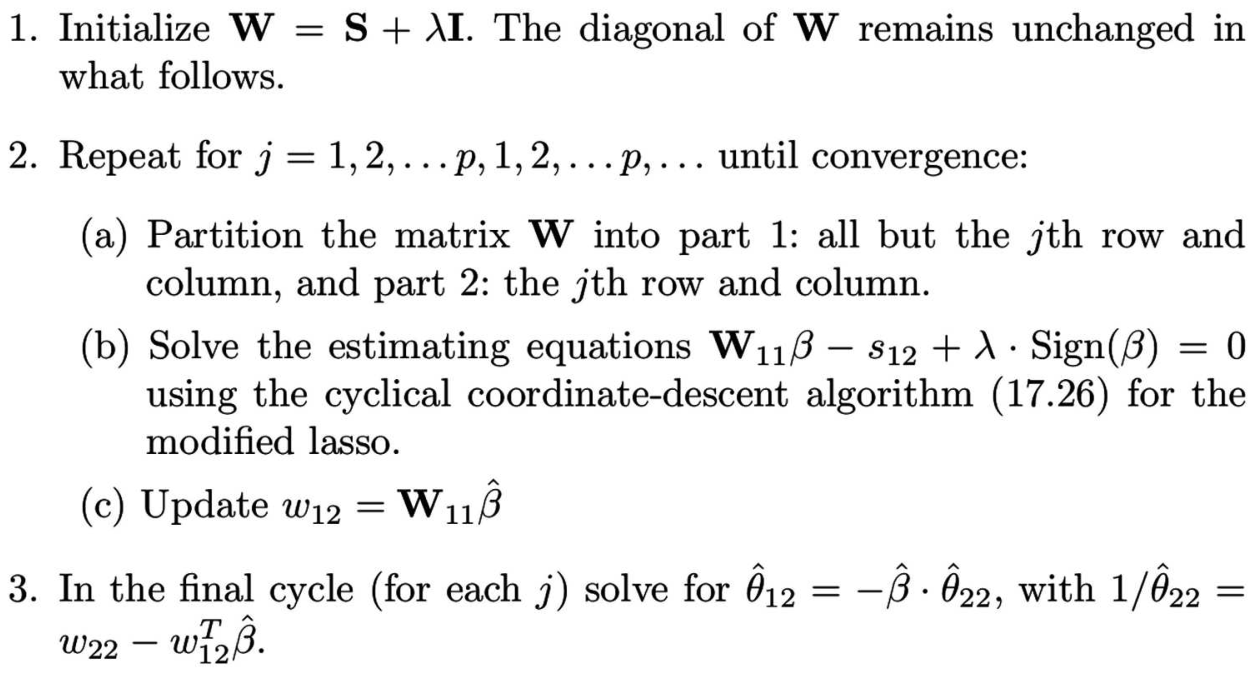}
	\caption{The glasso, exactly as it appears on p. 636 of
	\cite{Hastie09}. \label{Alg:glasso_alg1}}
\end{algorithm}

\subsection{Construction of $\hat{\Omega}_{\lambda}$ in the glasso}

The glasso iteratively updates a matrix $W$ which converges numerically
to $\hat{\Sigma}_{\lambda}$, the glasso estimator for the population
covariance matrix $\Sigma$. In contrast, the estimator $\hat{\Omega}_{\lambda}$
for the precision matrix $\Omega$ is constructed only upon convergence,
i.e., only after the algorithm terminates. As we shall show below,
the process by which $\hat{\Omega}_{\lambda}$ is constructed avoids
inversion but is however mathematically inexact in the sense that
it leads to $\hat{\Omega}_{\lambda}^{T}\neq\hat{\Omega}_{\lambda}$
and $\hat{\Omega}_{\lambda}^{-1}\neq\hat{\Sigma}_{\lambda}$ . If
$\hat{\Omega}_{\lambda}^{T}\neq\hat{\Omega}_{\lambda}$, the graph
encoded by the glasso output $\hat{\Omega}_{\lambda}$ may be different
from the graph encoded by $\hat{\Sigma}_{\lambda}^{-1}$. This problem
was illustrated in the two motivating examples above. 

Step 2 of Algorithm \ref{Alg:glasso_alg1} involves an inner loop in
which row/column $1,\dots,p$ of $W$ are sequentially updated. For
one full inner loop over the $p$ rows and columns of $W$, let the
$p$ successive estimates be denoted $W^{\left(i\right)}$ for $i=1,\dots,p$.
Exactly one row and column of $W^{\left(i\right)}$ is updated using
a lasso coefficient $\hat{\beta}^{\left(i\right)}$ ($\hat{\beta}$
of Step 2 in Algorithm \ref{Alg:glasso_alg1}).

We now introduce additional notation in order to illustrate the problems
encountered when the glasso constructs an estimate of the concentration
matrix (recall that this takes place upon termination of the glasso
algorithm). Consider once more $W^{\left(i\right)}$ for $i=1,\dots,p$.
Define $\Theta^{\left(i\right)}\triangleq\left(W^{\left(i\right)}\right)^{-1}$
and $\theta_{-i,i}^{\left(i\right)}$ to be the $\left(p-1\right)$
vector consisting of the $i^{th}$ column of $\Theta^{\left(i\right)}$
excluding the diagonal entry $\theta_{ii}^{\left(i\right)}$. Define
$w_{-i,i}^{\left(i\right)}$ and $w_{ii}^{\left(i\right)}$ be the
corresponding elements of $W^{\left(i\right)}$and let $W_{-i,-i}^{\left(i\right)}$
be the $i^{th}$ principal minor of $W^{\left(i\right)}$. Then using
the fact that $\Theta^{\left(i\right)}\triangleq\left(W^{\left(i\right)}\right)^{-1}$,
there is a closed-form expression for $\theta_{ii}^{\left(i\right)}$
and $\theta_{-i,i}^{\left(i\right)}$ in terms of $s_{ii}$, $w_{-i,i}^{\left(i\right)}$,
and $\hat{\beta}^{\left(i\right)}$: 

\begin{alignat}{3}
\label{eq:closedFormUpdates}
\theta_{-i,i}^{\left(i\right)} & =-\hat{\beta}^{\left(i\right)}\theta_{ii}^{\left(i\right)}, & \qquad \qquad
\theta_{ii}^{\left(i\right)} & =\frac{1}{w_{ii}^{\left(i\right)}-\left(w_{-i,i}^{\left(i\right)}\right)^{T}\hat{\beta}^{\left(i\right)}}.
\end{alignat}

When the glasso terminates, it sets $\hat{\Sigma}_{\lambda}=W^{\left(p\right)}$
and uses \eqref{eq:closedFormUpdates} to compute $\left\{ \theta_{ii}^{\left(i\right)},\theta_{-i,i}^{\left(i\right)}\right\} _{i=1}^{p}$,
which are taken as the columns of $\hat{\Omega}_{\lambda}$. This
procedure has a complexity of $\mathcal{O}\left(p^{2}\right)$ and
is therefore more efficient than direct numerical inversion.

\subsection{Cause of asymmetry in $\hat{\Omega}_{\lambda}$ }

The glasso terminates when $W$ converges numerically, and constructs
$\hat{\Omega}_{\lambda}$ from $\left\{ \theta_{-i,i}^{\left(i\right)},\theta_{ii}^{\left(i\right)}\right\} _{i=1}^{p}$.
These are easily obtainable from previous iterations of the inner
loop, thus avoiding the need to invert $W^{\left(p\right)}$. However,
while $\left\{ \theta_{-p,p}^{\left(p\right)},\theta_{pp}^{\left(p\right)}\right\} $
is equal to the the $p$th row and column of $\Theta^{\left(p\right)}=\left(W^{\left(p\right)}\right)^{-1}\triangleq\hat{\Sigma}_{\lambda}^{-1}$by
construction, the $\left\{ \theta_{ii}^{\left(i\right)},\theta_{i,-i}^{\left(i\right)}\right\} _{i=1}^{p-1}$
are not equal to the $i^{th}$ row and column of $\left(W^{\left(p\right)}\right)^{-1}$.
Instead, by construction each $\left\{ \theta_{ii}^{\left(i\right)},\theta_{i,-i}^{\left(i\right)}\right\} _{i=1}^{p-1}$
is equal to the $i^{th}$ row and column of $\left(W^{\left(i\right)}\right)^{-1}\neq\left(W^{\left(p\right)}\right)^{-1}$.
Asymmetry occurs because the quantities $\left\{ \theta_{ii}^{\left(i\right)},\theta_{i,-i}^{\left(i\right)}\right\} _{i=1}^{p}$
are taken as the columns of $\hat{\Omega}_{\lambda}$. 

The discrepancy between the above set of estimates may not be minimal
even if the iterates $W^{\left(i\right)}$ are approximately equal.
Another way of stating this problem is that convergence of the $W^{\left(i\right)}$
to a specified tolerance does not necessarily imply convergence of
$\left(W^{\left(i\right)}\right)^{-1}$ to any given tolerance. The
result is that while the glasso covariance estimator $\hat{\Sigma}_{\lambda}$
satisfies \eqref{eq:glassoPrimal}, $\hat{\Omega}_{\lambda}$ does
not, leading to the aforementioned problems. The problem is exacerbated
when the penalty parameter $\lambda$ is small and $S$ is close to
rank-deficient (which is the case when $n\ll p$). The following lemma
formalizes this assertion. 

\begin{lem}
	If $S$ is rank-deficient, the maximum absolute value of the entries
	of $\hat{\Omega}_{\lambda}$ diverges as $\lambda\rightarrow0$.
	\label{Lemma:IllCondition}
\end{lem}
\begin{proof}
See the appendix. 
\end{proof}

Lemma \eqref{Lemma:IllCondition} suggests that convergence of the inverse glasso iterates $W^{-1}$ to some small, fixed tolerance may require a radically small tolerance criterion for the convergence of the glasso iterates $W$. Indeed, it is easy to construct such examples. Consider the rank-deficient sample covariance matrix $S$ shown below alongside the optimal solution to \eqref{eq:glassoPrimal} corresponding to a regularization parameter of $\lambda = 10^{-6}$: 
\begin{alignat}{3}
S & =\left[\begin{array}{cc}
1 & 0\\
0 & 0
\end{array}\right], \ \ \ & \hat{\Sigma}_{\lambda} & =\left[\begin{array}{cc}
1+10^{-6} & 0\\
0 & 10^{-6}
\end{array}\right], \ \ \ & \hat{\Omega}_{\lambda} & =\left[\begin{array}{cc}
\left(1+10^{-6}\right)^{-1} & 0\\
0 & 10^{6}
\end{array}\right].
\end{alignat}
Moreover, consider a matrix iterate $W_{t}$ which is close to $\hat{\Sigma}_{\lambda}$, given as follows
\begin{alignat}{1}
W_{t} & =\left[\begin{array}{cc}
1+10^{-6} & 0\\
0 & (1+t^{-1})\times 10^{-6}
\end{array}\right].
\end{alignat}
The supremum-norm errors on the dual and primal for this $W_t$ are given respectively by
\begin{alignat}{2}
\left\Vert W_{t}-\hat{\Sigma}_{\lambda}\right\Vert _{\infty} & =t^{-1} 10^{-6} \ \ \mbox{  and  } \ \ 
\left\Vert W_{t}^{-1}-\hat{\Omega}_{\lambda}\right\Vert _{\infty} & = \frac{1}{t+1}10^{6},
\end{alignat}
and are several orders of magnitude apart. This example demonstrates why decreasing the convergence tolerance on the dual iterates $W$ may not always be a feasible solution to the asymmetry problem discussed here. 

To summarize, the method for inversion used during the final step of the glasso algorithm for computing $\hat{\Omega}_\lambda$ is mathematically inexact, and the resulting error is exacerbated when $p>n$ with an insufficiently large penalty parameter $\lambda$. In this case, the $\ell_1$-penalized inverse covariance estimator is unreliable as $\lambda \rightarrow 0$, as described by Lemma \ref{Lemma:IllCondition} Therefore the use of an overly small $\lambda$ should be avoided; however, in practice, choosing the penalty parameter $\lambda$ can be challenging. For example, choosing $\lambda$ via cross-validation using $\hat{\Omega}_\lambda$ tends to yield overly small $\lambda$, which may produce dense and possibly ill-conditioned estimates for $\hat{\Omega}_\lambda$. One possible indicator of too little regularization is when the number of neighbors of each variable/node is too high. A second possible indicator is if there is a serious lack of symmetry in the glasso estimates. A third possible indicator is  when the condition number of the resulting estimate is too high. Recently, \cite{Won2012Condition} provide impetus for constraining the condition number of the covariance matrix; in light of that work, the condition number can perhaps be used as a guide in choosing $\lambda$.

\section{Enforcing Symmetry on the Glasso Concentration Matrix Estimator}
\label{Sec:solutions}

The glasso covariance estimator $\hat{\Sigma}_{\lambda}$ is the true
numeric minimum of the glasso problem \eqref{eq:glassoPrimal} and
thus a valid $\ell_{1}$ regularized estimator for the true population
covariance matrix $\Sigma$. However, as previously demonstrated,
the glasso estimator $\hat{\Omega}_{\lambda}$ is asymmetric, and
$\hat{\Omega}_{\lambda}^{-1}\neq\hat{\Sigma}_{\lambda}$.

In some settings, it may be desirable to resolve one or both of the
aforementioned issues. For $\hat{\Omega}_{\lambda}$ to encode a sparse
concentration graph, its sparsity pattern must be symmetric. Moreover,
if $\hat{\Omega}_{\lambda}$ is to be used as a sparse concentration
matrix estimator, it is necessary that $\hat{\Omega}_{\lambda}^{T}=\hat{\Omega}_{\lambda}$
for it to be a valid estimator. Most importantly, it may be required
that $\hat{\Omega}_{\lambda}=\hat{\Sigma}_{\lambda}^{-1}\succ0$ in
order for it to be usable in multivariate procedures. 

We propose three simple approaches which address some or all of the
above requirements. 

\begin{enumerate}
	\item \textbf{Numerical inversion. } To have $\hat{\Omega}_{\lambda}=\hat{\Sigma}_{\lambda}^{-1}$,
		it is necessary to directly invert $\hat{\Sigma}_{\lambda}$. This
		inversion maintains the sparsity pattern of $\hat{\Omega}_{\lambda}$(although
		as a consequence of numerical error there may be negligible entries
		in place of zeroes). The $\mathcal{O}(p^3)$ complexity of numerical inversion (vs. $\mathcal{O}(p^2)$ for the current glasso approach) does not represent a difficulty for matrices of the dimension which the glasso is currently able to solve (up to $p \approx 2000$ on a typical desktop), as it also needs to be done only once at the end of the glasso iteration. Note that the inverted matrix $\hat{\Sigma}_{\lambda}^{-1}$ should be hard thresholded to eliminate small non-zero entries introduced by numerical error, and the resulting inverse covariance matrix then should be checked for positive definiteness. However, numerical inversion may not be a viable option when $\hat{\Sigma}_{\lambda}$ is ill-conditioned, as is the case when $S$ is rank-deficient and
$\lambda$ is small. In such cases, it may be useful to exercise caution when using $\hat{\Omega}_\lambda$ in further calculations.

	\item \textbf{Modified glasso output. }The upper right triangle of $\hat{\Omega}_{\lambda}$
		can be taken as the correct estimate. The entries corresponding to
		the upper right triangle are more recent updates than those in the
		lower left triangle, since the glasso inserts the $\left\{ \theta_{-i,i}^{\left(i\right)},		\theta_{ii}^{\left(i\right)}\right\} _{i=1}^{p}$
		into the columns of $\hat{\Omega}_{\lambda}$. The resulting estimator will not equal
		$\hat{\Sigma}_{\lambda}^{-1}$, but it is symmetric. It will not solve
		the primal problem in \eqref{eq:glassoPrimal} exactly. 
			
		\item \textbf{Iterative proportional fitting} \textbf{(IPF)}. IPF \citep{Speed86}
			can be used to simultaneously compute the maximum likelihood estimates
			for $\Omega$ and $\Sigma$ under an assumed concentration graph,
			i.e., sparsity pattern in $\hat{\Omega}$. One approach is to use
			the sparsity pattern from the upper right triangle of $\hat{\Omega}_{\lambda}$,
			enforce symmetry, and then use IPF to obtain $\hat{\Sigma}$ and $\hat{\Omega}$.
			The estimator $\hat{\Omega}$ will reflect the sparsity structure corresponding
			to $\hat{\Omega}_{\lambda}$, and satisfy $\hat{\Omega}=\hat{\Sigma}^{-1}$
			at each iteration of IPF. Note that neither $\hat{\Omega}$ nor $\hat{\Sigma}$ will be solutions to \eqref{eq:glassoPrimal} or (\ref{eq:dualProblem}), respectively. Furthermore, the computational complexity of IPF is $\mathcal{O}(c^3)$, where $c$ is the size of the largest maximal clique of the graph implied by $\hat{\Omega}_{\lambda}$. Therefore, IPF does not imply relatively higher computational costs, although it does require identifying the maximal cliques which is well-known to be NP-complete. Finding the maximal cliques can however be avoided if a modification of the glasso algorithm is used to estimate an undirected Gaussian graphical model with known structure (see Algorithm 17.1 in \cite{Hastie09}).			
\end{enumerate}

\color{black}
Table \ref{tab:comparison_tab1} summarizes the properties and tradeoffs
of each of the proposed solutions. 

% Table 1, comparison_tab1, goes here. 

\begin{table}[H]
	\caption{\label{tab:comparison_tab1}Comparison of possible estimators.}
	\begin{tabular}{cccccc}
		\hline 
		Method & $\hat{\Omega}^{T}=\hat{\Omega}$ & Latest updates in $\hat{\Omega}$ & 			 	$\hat{\Omega}=\hat{\Sigma}^{-1}$ & $\hat{\Sigma}$ solves (\ref{eq:dualProblem}) & $\hat{\Omega}$ solves (\ref{eq:glassoPrimal})\tabularnewline
		\hline 
		Glasso Output & $X$ & $X$ & $X$ & $\checkmark$ & $X$\tabularnewline
		
		Modified Output & $\checkmark$ & $\checkmark$ & $X$ & $\checkmark$ & $X$\tabularnewline
		
		Numerical Inversion & $\checkmark$ & $\checkmark$ & $\checkmark$ & $\checkmark$ & 		$\checkmark$\tabularnewline
		
		IPF & $\checkmark$ & $X$ & $\checkmark$ & $X$ & $X$\tabularnewline
		\hline 
	\end{tabular}
\end{table}

\section{Conclusions}

In this note we demonstrated that the estimators from the widely used \emph{R} package \emph{glasso} may be asymmetric when the amount of regularization applied is small. This could cause problems when the glasso estimators are used as inputs to other multivariate procedures, and additionally because the sparsity structure of the glasso estimators may themselves be asymmetric. It may be helpful for users of the package \emph{glasso} to be aware of this, as the estimator can be easily corrected by one of the outlined methods. Of these, numerical inversion followed by thresholding may be the simplest and most effective fix. The root cause of the issue is that the glasso algorithm operates on the dual of \eqref{eq:glassoPrimal}, and constructs the primal estimator, $\hat{\Omega}_\lambda$, only after the dual optimization completes. If a sparse concentration estimator is sought, it may be more natural to operate off the primal problem \eqref{eq:glassoPrimal}, though the glasso is more popular in practice. Methods for solving the primal \eqref{eq:glassoPrimal} have been recently considered, among others see \cite{Maleki2010} and \cite{MazumderPrimal2011}. This short note avoids recourse to the primal by identifying problems with the dual approach, and consequently explores ways in which these can be easily rectified so that the popular dual approach can be retained. 

\section*{Acknowledgments}

We acknowledge Trevor Hastie and Robert Tibshirani (Department of Statistics, Stanford University) for discussions.

Benjamin Rolfs was supported in part by the Department of Energy Office of Science Graduate Fellowship Program DE-AC05-06OR23100 (ARRA) and NSF grant AGS1003823. Bala Rajaratnam was supported in part by NSF grants DMS0906392 (ARRA), AGS1003823, DMS (CMG) 1025465, DMS1106642 and grants NSA H98230-11-1-0194, DARPA-YFA N66001-11-1-4131 and SUFSC10-SUSHSTF09-SMSCVISG0906. 

\bibliographystyle{model2-names}
\bibliography{refs_btrolfs}

\section*{Appendix}
\begin{proof}
	Consider the dual of \eqref{eq:glassoPrimal} as given in \cite{Banerjee08}:
	
	\begin{alignat}{1}
		\hat{\Sigma}_{\lambda} & =\arg\min_{X\succ0}\ 	\left[\log\det\left(X\right)\right]\label{eq:dualProblem}\\
		 & \mbox{s.t. }\max_{i,j}\left|x_{ij}-s_{ij}\right|\leq\lambda\nonumber 
	\end{alignat} 
	where $\max_{i,j}\left|m_{ij}\right|$ is the supremum norm, the maximum
	absolute value entry of the matrix $M$. From \eqref{eq:dualProblem},
	it is clear that $\hat{\Sigma}_{\lambda}\rightarrow S$ in the supremum
	norm as $\lambda\rightarrow0$, though at $\lambda=0$ the primal
	problem \eqref{eq:glassoPrimal} does not necessarily have a solution.
	Convergence in $\sup$-norm gives convergence of $\hat{\Sigma}_{\lambda}\rightarrow S$
	in any other operator norm $\left\Vert \bullet\right\Vert _{\star}$.
	In particular, invoking the continuity of eigenvalues, 	$\lambda_{\min}\left(\hat{\Sigma}_{\lambda}\right)\rightarrow\lambda_{\min}\left(S\right)$
	as $\lambda\rightarrow0$, with $\lambda_{\min}\left(M\right)$
	defined as the smallest eigenvalue of the square matrix $M$. Considering
	the operator $2$-norm and $\infty$-norm of $\hat{\Sigma}_{\lambda}^{-1}$
	gives: 
	\begin{alignat*}{1}
		\max_{i,j}\left|\left(\hat{\Omega}_{\lambda}\right)_{ij}\right| & =\max_{i,j}\left|		\left(\hat{\Sigma}_{\lambda}^{-1}\right)_{ij}\right|\\
	 	 & \geq p^{-1}\left\Vert \hat{\Sigma}_{\lambda}^{-1}\right\Vert _{\infty}\\
		 & \geq p^{-1}\left\Vert \hat{\Sigma}_{\lambda}^{-1}\right\Vert _{2}\\
		 & =p^{-1}\lambda_{\max}\left(\hat{\Sigma}_{\lambda}^{-1}\right)\\
		 & =p^{-1}\left[\lambda_{\min}\left(\hat{\Sigma}_{\lambda}\right)\right]^{-1}\\
		 & \underset{\lambda\rightarrow0}{\longrightarrow}p^{-1}\left[\lambda_{\min}\left(S\right)\right]^{-1}
	\end{alignat*}
	In the sample-deficient case $n\ll p$, $\lambda_{\min}\left(S\right)=0$
	almost surely, and therefore $\hat{\Omega}_{\lambda}$ diverges with respect to the supremum norm as  $\lambda\rightarrow0$. 
\end{proof}

\end{document}